\PassOptionsToPackage{table}{xcolor}
\documentclass[11pt]{article}

\usepackage[preprint]{acl}
\usepackage{times}
\usepackage{latexsym}
\usepackage{array}
\usepackage{tabularx}
\usepackage{hyperref}
\usepackage[utf8]{inputenc}
\usepackage[T1]{fontenc}
\usepackage{inconsolata}
\usepackage[table]{xcolor}
\usepackage{pifont}

\newcommand{\rwcheck}{\textcolor{green!55!black}{\ding{51}}}
\newcommand{\rwxmark}{\textcolor{red!70!black}{\ding{55}}}
\usepackage{microtype}
\usepackage{url}
\usepackage{booktabs}
\usepackage{graphicx}
\usepackage{subcaption}
\usepackage{multirow}
\usepackage{enumitem}
\usepackage{makecell}
\usepackage{wrapfig}

\usepackage{amsmath}
\usepackage{amssymb}
\usepackage{amsfonts}
\usepackage{mathtools}
\usepackage{amsthm}

\usepackage{xcolor}

\theoremstyle{plain}
\newtheorem{theorem}{Theorem}[section]

\newtheorem{lemma}[theorem]{Lemma}

\theoremstyle{definition}

\newtheorem{assumption}[theorem]{Assumption}

\theoremstyle{remark}

\definecolor{darkblue}{rgb}{0,0,0.5}
\hypersetup{
  colorlinks=true,
  citecolor=darkblue,
  linkcolor=darkblue,
  urlcolor=darkblue
}

\title{ActTail: Global Activation Sparsity in Large Language Models}

\author{
Wenwen Hou\textsuperscript{1,*}\quad Xinyuan Song\textsuperscript{2,*}\quad
Wenqian Weng\textsuperscript{3}\quad Shiwei Liu\textsuperscript{4}\\
\textsuperscript{1}The Hong Kong University of Science and Technology (Guangzhou)\\
\textsuperscript{2}Emory University\quad \textsuperscript{3}Wayne State University\\
\textsuperscript{4}ELLIS Institute T\"ubingen \& Max Planck Institute for Intelligent Systems\\
\textsuperscript{*}Equal contribution.\quad Correspondence: \texttt{sliu@tue.ellis.eu}
}

\begin{document}

\maketitle

\begin{abstract}
  Activation sparsity is a promising approach for accelerating large language model (LLM) inference by reducing computation and memory movement. However, existing activation sparsity methods typically apply uniform sparsity across projections, ignoring the heterogeneous statistical properties of Transformer weights and thereby amplifying the performance degradation. In this paper, we propose \textbf{ActTail}, a TopK magnitude-based \textbf{act}ivation sparsity method
with global activation sparsity allocation grounded in Heavy-\textbf{Tail}ed Self-Regularization (HT-SR) theory. Specifically, we capture this heterogeneity via the heavy-tail exponent $\alpha$ computed from each projection’s empirical spectral density (ESD), which is used as a quantitative indicator to assign projection-specific sparsity budgets. Importantly, we provide a theoretical analysis that establishes an explicit relationship between the activation sparsity ratio and the heavy-tail exponent $\alpha$ under the HT-SR regime, offering principled guidance for sparsity allocation beyond heuristic design. Experiments on LLaMA and Mistral models show that our method improves both perplexity and downstream task performance at high sparsity compared to uniform allocation. At 80\% sparsity, perplexity is reduced by 21.8\% on LLaMA-2-7B, 40.1\% on LLaMA-2-13B, and 9.4\% on Mistral-7B.
\end{abstract}

\section{Introduction}
Modern large language models have grown to tens or even hundreds of billions of parameters, leading to substantial memory and computational overhead during inference and posing challenges for deployment in settings with limited resources\citep{kaplan2020scalinglawsneurallanguage,NEURIPS2022_c1e2faff,touvron2023llama,yang2025qwen3}. To address these challenges, a broad class of efficiency techniques including weight pruning\citep{frantar2023sparsegpt,an2024fluctuationbased,ashkboos2024slicegpt,ling2024slimgpt,sun2024a,shen2025numerical}, and low rank decomposition\citep{wang2025svdllm,Huang_Huang_Wen_2025} has been proposed to reduce model size and computational cost. Weight pruning typically eliminates parameters or units(e.g., neurons or channels) deemed unimportant according to predefined importance metrics, resulting in sparse weight matrices.
Low-rank decomposition, on the other hand, reduces model complexity by factorizing weight matrices into low-rank components. Once applied, these methods induce a static model structure, whereby all input instances are processed using the same sparsity pattern or low-rank factorization. Therefore, these methods often suffer from accuracy loss and reduced robustness at high compression rates, which limits their effectiveness in practice. Recently, a complementary line of work on contextual activation sparsity has 
emerged\citep{liu2023dejaa,lee2024cats,zhang2025rsparse,liu2025rosa,liu2025trainingfree} inspired by lazy neuron phenomenon, which exploits input-dependent sparsity in hidden states. Because channels associated with zero valued activations do not participate in computation, both arithmetic operations and memory movement can be reduced by skipping the corresponding weight channels during memory access, enabling efficient inference under channel wise sparsity.

Despite their success, most activation sparsity methods \citep{liu2025trainingfree,liu2025rosa} rely on either uniform rules or greedy search to determine sparsity rates across layers and projections. Uniform sparsity ignores the heterogeneous importance of different components and often leads to unnecessary accuracy loss at high sparsity levels, while greedy procedures lack 
\begin{figure}[t]
    \centering
    \includegraphics[width=\linewidth]{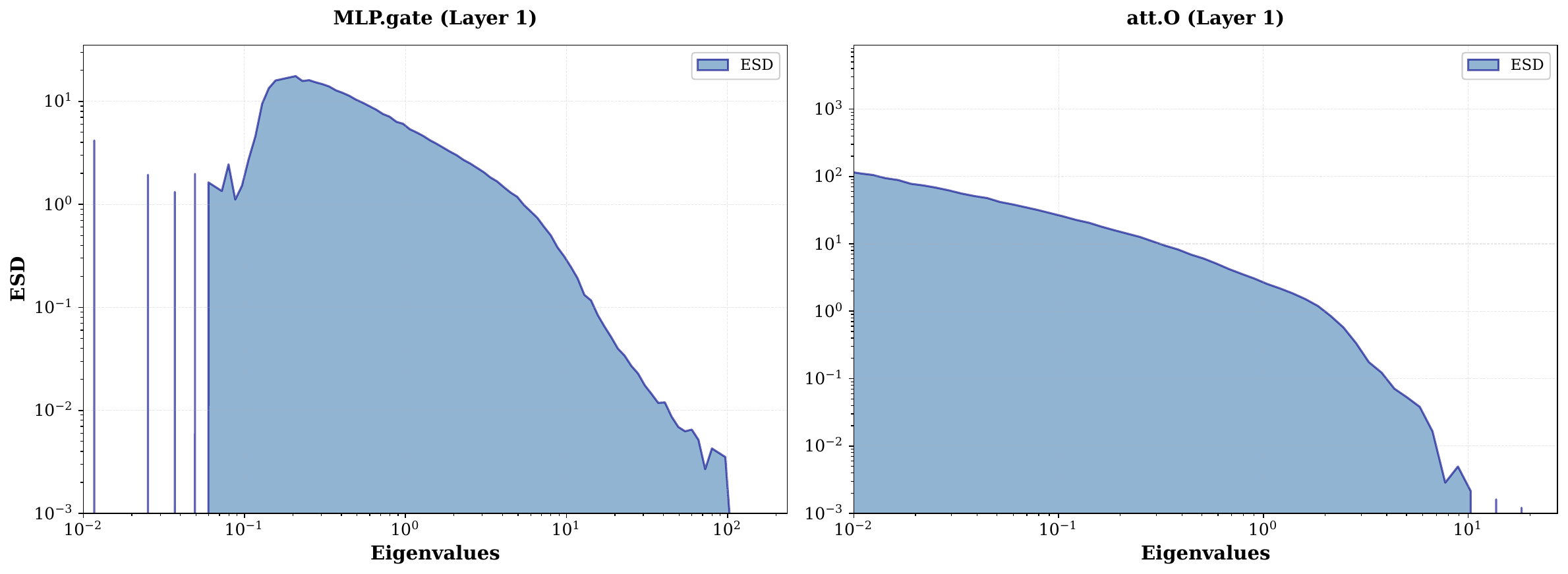}
    \caption{ESDs of the MLP gate and attention output projections in Layer 1 of Llama2-13B.}
    \label{fig:esd_layer1_mlp_gate_attn_o}
\end{figure}

theoretical guidance and incur substantial computational overhead. To derive more effective and theoretically guided sparsity allocation, we leverage insights from Heavy Tailed Self Regularization theory \citep{martin2017rethinking,martin2018implicit_JMLRversion,martin2019traditional,martin2020heavy,he2025alphadecay}.

HT-SR theory studies trained weight matrices through empirical spectral densities (ESDs) of the corresponding correlation matrices (see Figure~\ref{fig:esd_layer1_mlp_gate_attn_o}), connecting spectral structure to matrix-level training quality. Building on this property, we develop an HT-SR theory guided activation sparsity scheme.

This work focuses on input activation sparsity in Transformer projections, including the linear projections in self-attention and the feed-forward network (FFN). Input sparsity is induced using topK selection on the layer input hidden state, which keeps only the K largest entries and sets the rest to zero. Hence, the corresponding weight channels are unused and can be skipped. The heavy tail exponent $\alpha$ from Heavy Tailed Self Regularization theory is used to guide sparsity at the projection level. Specifically, each projection’s $\alpha$ is mapped to a projection specific sparsity ratio, which determines its topK retention rate during inference. In summary, the contributions are as follows:
\begin{itemize}[left = 0em]
    \item A data-independent activation sparsity method based on topK sparsification is introduced, which induces input activation sparsity without relying on calibration data.

    \item We establish a theoretical connection between Heavy Tailed Self Regularization and activation sparsity, showing how heavy tailed spectral statistics predict the sparsifiability of different projections and support projection level sparsity allocation.

    \item Extensive experiments on the LLaMA2 family, Mistral 7B, and Qwen1.5 7B demonstrate that ActTail consistently outperforms uniform allocation, improving both perplexity and downstream task performance at high sparsity.
\end{itemize}
Overall, ActTail offers a principled and practical framework for activation sparsity in large language models, combining theoretical insights from heavy-tailed self-regularization with efficient topK inference-time sparsification. By adapting sparsity to the intrinsic statistical properties of individual projections, ActTail achieves substantial efficiency gains while maintaining strong performance across perplexity and downstream benchmarks.

\begin{table*}[!ht]
\centering
\footnotesize
\setlength{\tabcolsep}{3pt}
\caption{Comparison with closely related sparsification methods.}

\label{tab:related_work_comparison}

\resizebox{\textwidth}{!}{%
\begin{tabular}{
    @{}
    >{\raggedright\arraybackslash}p{0.30\textwidth}
    >{\centering\arraybackslash}p{0.12\textwidth}
    >{\centering\arraybackslash}p{0.12\textwidth}
    >{\centering\arraybackslash}p{0.12\textwidth}
    >{\centering\arraybackslash}p{0.12\textwidth}
    >{\centering\arraybackslash}p{0.12\textwidth}
    >{\centering\arraybackslash}p{0.12\textwidth}
    @{}
}
\toprule
\textbf{Work}
& \textbf{Weight Spectra}
& \textbf{Global Budget}
& \textbf{Calibration-Free}
& \textbf{Training-Free}
& \textbf{Activation TopK}
& \textbf{Spectral Theory} \\
\midrule

TEAL~\citep{liu2025trainingfree}
& \rwxmark & \rwxmark & \rwxmark
& \rwcheck & \rwxmark & \rwxmark \\

WAS~\citep{wang2025was}
& \rwxmark & \rwcheck & \rwxmark
& \rwcheck & \rwxmark & \rwxmark \\

WINA~\citep{chen2026wina}
& \rwxmark & \rwxmark & \rwxmark
& \rwcheck & \rwcheck & \rwxmark \\

WiSparse~\citep{chen2026wisparse}
& \rwxmark & \rwcheck & \rwxmark
& \rwcheck & \rwxmark & \rwxmark \\

TopK-Guided~\citep{agarwalla2026topkguided}
& \rwxmark & \rwxmark & \rwxmark
& \rwcheck & \rwcheck & \rwxmark \\

LaRoSA~\citep{liu2025rosa}
& \rwxmark & \rwxmark & \rwxmark
& \rwcheck & \rwcheck & \rwxmark \\

AlphaPruning~\citep{lu2024alphapruninga}
& \rwcheck & \rwcheck & \rwcheck
& \rwcheck & \rwxmark & \rwxmark \\

AlphaDecay~\citep{he2025alphadecay}
& \rwcheck & \rwxmark & \rwxmark
& \rwxmark & \rwxmark & \rwxmark \\

\midrule
\rowcolor{blue!9}
\textbf{ActTail (ours)}
& \rwcheck & \rwcheck & \rwcheck
& \rwcheck & \rwcheck & \rwcheck \\

\bottomrule
\end{tabular}%
}

\end{table*}

\section{Related Work}

\paragraph{Activation Sparsity in LLMs.}
Activation sparsity arises when a substantial fraction of a model’s hidden-state entries are zero, and can be viewed as a form of conditional computation that selectively activates parts of the network\citep{bengio2016conditionalcomputationneuralnetworks}.  Such sparsity naturally emerges in the intermediate representations of ReLU-based MLP blocks\citep{li2023the}. Prior work has demonstrated that activation sparsity can be exploited to accelerate LLM inference by avoiding unnecessary computation and memory transfer. \citet{liu2023dejaa} leverages zero-valued activations to skip transferring corresponding weight channels to GPU registers during decoding, while \citet{song2024powerinfer} and \citet{alizadeh2024llm} extend this idea to CPU offloading, significantly reducing data movement between CPU and GPU memory. However, modern leading LLMs increasingly adopt non-ReLU activation functions, such as SwiGLU\citep{so2021searching}, which do not inherently produce sparse activations\citep{touvron2023llama,qwen1.5}.

Recent work focuses on inducing activation sparsity in modern LLM architectures that no longer rely on ReLU-based MLPs. \citet{mirzadehrelu} reintroduce sparsity by replacing SiLU or GeLU with ReLU followed by continued large-scale pretraining, while \citet{zhang2024relu2} systematically study alternative activations and identify Squared ReLU as an effective replacement. In the training-free setting, \citet{lee2024cats} induce activation sparsity in SwiGLU-based LLMs by applying magnitude-based thresholding to the gating output. \citet{liu2025trainingfree} extends magnitude-based activation sparsity to the input activations of both self-attention and MLP blocks. More recently, \citet{liu2025rosa} introduce layerwise orthogonal rotations to transform input activations into a space that is more amenable to sparsification, enabling consistent model-level sparsity and reliable wall-clock acceleration without additional training. \citet{shrestha2025polar} propose Polar Sparsity, which reveals a shift in sparsity efficiency from MLP layers to attention layers as batch size and sequence length increase.

\paragraph{Heavy-Tailed Self-Regularization Theory.}
Heavy-Tailed Self-Regularization (HT-SR) theory studies the empirical spectral density (ESD) of trained neural network weight matrices and relates their spectral structure to training quality, drawing on insights from statistical physics and random matrix theory\citep{martin2018implicit_JMLRversion,couillet2022random}. Under HT-SR theory, heavier-tailed ESDs of weight matrices indicate stronger optimization-induced correlations and higher layer training quality \citep{martin2019traditional}. Recently, HT-SR has been successfully applied across a range of model analysis and optimization tasks\citep{lu2024alphapruninga,he2025alphadecay}. In model selection, HT-SR metrics have been shown to reliably rank pretrained networks and checkpoints by quality without access to training data \citep{martin2021predictinga,yang2023test}.
In module-wise adaptive training, heavy-tailed spectral statistics have been used to diagnose layer-wise training imbalance and to guide adaptive optimization strategies\citep{zhou2023temperature}. 

\paragraph{Comparison to Prior Work.}
Table~\ref{tab:related_work_comparison} highlights ActTail's distinction from prior methods through calibration-free spectral allocation, global activation TopK sparsification, and theoretical guarantees linking spectral properties to truncation error.

\begin{figure*}[t]
  \centering
  \includegraphics[
    width=\textwidth,
    height=1\textheight,
    keepaspectratio,
    trim=0 80 0 60,
    clip
  ]{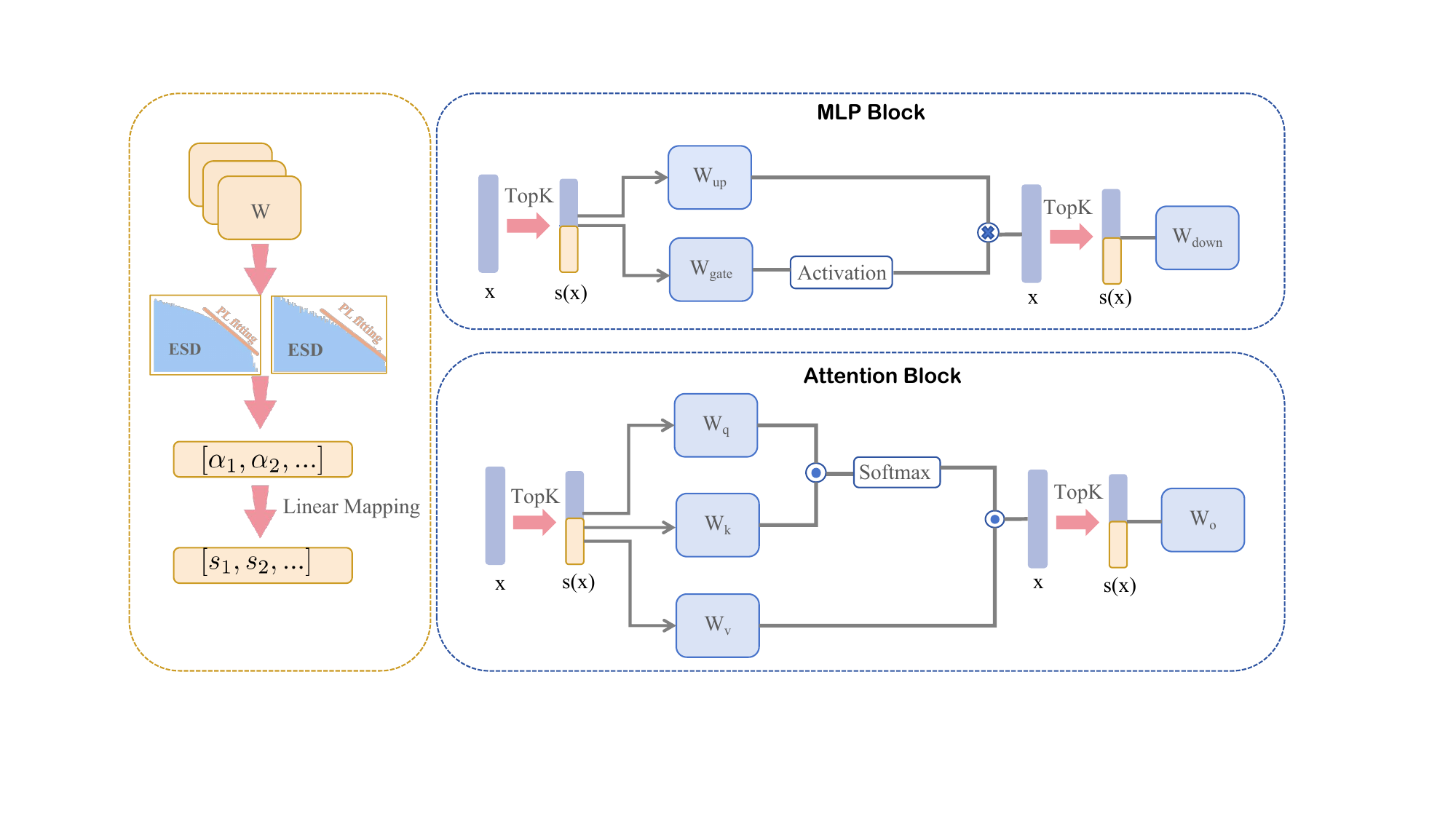}
    
  \caption{Overview of the ActTail workflow. ActTail first estimates heavy-tailed exponents from the empirical spectral density (ESD) of each projection, maps them to projection-level sparsity ratios via a linear rule, and applies TopK activation sparsity during inference.}
  \label{fig:acttail_workflow}
\end{figure*}

\section{Methodology}
\subsection{Notation}

We consider a Transformer with $L$ layers and focus on inference-time acceleration at the decoding stage.
At each decoding step, for each linear projection $p$ in layer $\ell$, the projection receives an input activation
$\mathbf h^{(\ell,p)} \in \mathbb{R}^{d_{\text{in}}^{(p)}}$.
The set of projections is indexed by
$p \in \mathcal P=\{\text{q},\text{k},\text{v},\text{o},\text{gate},\text{up},\text{down}\}$.

Each projection is parameterized by a weight matrix
$\mathbf W^{(\ell,p)} \in \mathbb{R}^{d_{\text{out}}^{(p)} \times d_{\text{in}}^{(p)}}$.
Input activation sparsity is applied during decoding via a binary mask
$\mathbf m^{(\ell,p)} \in \{0,1\}^{d_{\text{in}}^{(p)}}$,
producing a sparsified projection input
$\tilde{\mathbf h}^{(\ell,p)}=\mathbf m^{(\ell,p)}\odot\mathbf h^{(\ell,p)}$
and the corresponding projection output
$\mathbf a^{(\ell,p)}=\mathbf W^{(\ell,p)}\tilde{\mathbf h}^{(\ell,p)}$.
The sparsity ratio $s^{(\ell,p)}\in[0,1]$ is defined as the fraction of zero entries in $\mathbf m^{(\ell,p)}$.

\noindent
\begin{minipage}[t]{\linewidth}

\scriptsize
\setlength{\tabcolsep}{3pt}
\begin{tabular}{lcccc}
\toprule
\hline
\textbf{Method} &
\textbf{7B} &
\textbf{13B} &
\textbf{7B} &
\textbf{7B} \\
& \textbf{LLaMA2} & \textbf{LLaMA2} & \textbf{Mistral} & \textbf{Qwen1.5} \\
\midrule
Dense (0\%)      
& 5.11 & 4.57 & 4.91 & 7.36 \\
\midrule
\textsc{TEAL}$_{\text{Top-}K}$ (60\%)   
& 5.94 & 5.15 & 5.57 & 8.47 \\
ActTail (60\%)   
& \textbf{5.86} & \textbf{5.14} & \textbf{5.57} & \textbf{8.42} \\
\midrule
\textsc{TEAL}$_{\text{Top-}K}$ (70\%)   
& 9.35 & 6.88 & 7.84 & 12.28 \\
ActTail (70\%)   
& \textbf{8.09} & \textbf{6.25} & \textbf{7.70} & \textbf{11.40} \\
\midrule
\textsc{TEAL}$_{\text{Top-}K}$ (80\%)   
& 36.37 & 29.66 & 23.87 & 57.24 \\
ActTail (80\%)   
& \textbf{28.43} & \textbf{17.76} & \textbf{21.63} & \textbf{56.21} \\
\hline
\bottomrule
\end{tabular}

\captionof{table}{Perplexity on WikiText2. Lower is better.}
\label{tab:ppl_main}
\end{minipage}
\par\medskip\noindent
\begin{minipage}[t]{\linewidth}

For each projection $p$ in layer $\ell$, the correlation matrix is 
$\mathbf X^{(\ell,p)}=\mathbf W^{(\ell,p)\top}\mathbf W^{(\ell,p)}$.
Let $\{\lambda_j\}_{j=1}^{n}$ be the eigenvalues and the empirical spectral density (ESD) is defined as 
$\mu_{\mathbf X^{(\ell,p)}}:=\frac{1}{n}\sum_{j=1}^{n}\delta_{\lambda_j}$.

For the ESD $\mu_{\mathbf X^{(\ell,p)}}$,
\begin{equation}
p(\lambda)\propto \lambda^{-\alpha^{(\ell,p)}}, 
\quad 
\lambda_{\min}^{(\ell,p)}<\lambda<\lambda_{\max}^{(\ell,p)},
\end{equation}
is the fitted power-law density function.
\end{minipage}

\subsection{TopK Input Activation Sparsity during Inference}
\label{sec:topk_input_sparsity}
In this work, we focus on input activation sparsity for the decoding stage, targeting the linear projections in each Transformer layer.
For projection $p$ in layer $\ell$, a binary mask $\mathbf m^{(\ell,p)}$ is applied to the hidden state $\mathbf h^{(\ell)}$ before the projection.
If an entry of $\mathbf h^{(\ell)}$ is zero, then the corresponding column of $\mathbf W^{(\ell,p)}$ is unused and can be skipped, reducing both computation and weight memory access.
Formally,
\begin{equation}
\mathbf a^{(\ell,p)} = \mathbf W^{(\ell,p)} (\mathbf m^{(\ell,p)}\odot\mathbf h^{(\ell)})
=\mathbf W^{(\ell,p)} \tilde{\mathbf h}^{(\ell,p)},
\end{equation}
where the sparsity ratio $s^{(\ell,p)}$ is the fraction of zeros in $\mathbf m^{(\ell,p)}$.

The mask $\mathbf m^{(\ell,p)}$ is constructed in a magnitude based way.
Prior work such as \citep{liu2025trainingfree} uses thresholding, which zeros out entries whose magnitudes fall below a learned threshold. However, it depends on calibration data and does not guarantee an exact sparsity ratio.
Instead, we propose a TopK rule to build $\mathbf m^{(\ell,p)}$ by keeping the $K^{(\ell,p)}$ largest-magnitude entries of $\mathbf h^{(\ell)}$ and masking out the rest.
Unlike thresholding, TopK fixes the number of active dimensions, so the sparsity ratio is explicitly controlled as
\begin{equation}
s^{(\ell,p)} = 1-\frac{K^{(\ell,p)}}{d_{\text{in}}^{(p)}} .
\end{equation}
This provides stable, data-independent sparsity control during inference.

\subsection{Heavy-Tailed Self-Regularization Theory (HT-SR)}
\paragraph{Motivation: from HT-SR theory to activation sparsity.}
Given a projection weight matrix $\mathbf W^{(\ell,p)}$, HT-SR studies the empirical spectral density (ESD) of its correlation matrix $\mathbf \Sigma^{(\ell,p)}=\mathbf W^{(\ell,p)\top}\mathbf W^{(\ell,p)}$. HT-SR theory links the spectral structure of trained weight matrices to the quality of learned representations \citep{MM21a_simpsons_TR,martin2020heavy,martin2021predictinga,he2025alphadecay}. 
This connection arises because the ESD reflects the correlation structure encoded in $\mathbf W^{(\ell,p)}$.
Effective training induces stronger correlations, which manifest as heavy-tailed spectral behavior and indicate higher matrix training quality.

This perspective provides a natural motivation for input activation sparsity.
Input sparsification removes a subset of hidden-state dimensions before multiplying by $\mathbf W^{(\ell,p)}$, which is equivalent to skipping the corresponding columns of $\mathbf W^{(\ell,p)}$ during decoding.
According to HT-SR theory, projections with more heavy-tailed ESDs tend to encode stronger learned structure, so they should be sparsified more conservatively.
Conversely, projections with lighter-tailed spectra can tolerate higher input sparsity.
Therefore, HT-SR theory offers a principled signal for allocating activation sparsity across projections, replacing uniform rules or costly search with a theory-guided allocation strategy.

\begin{figure}[!ht]
    \centering
    \includegraphics[width=\linewidth]{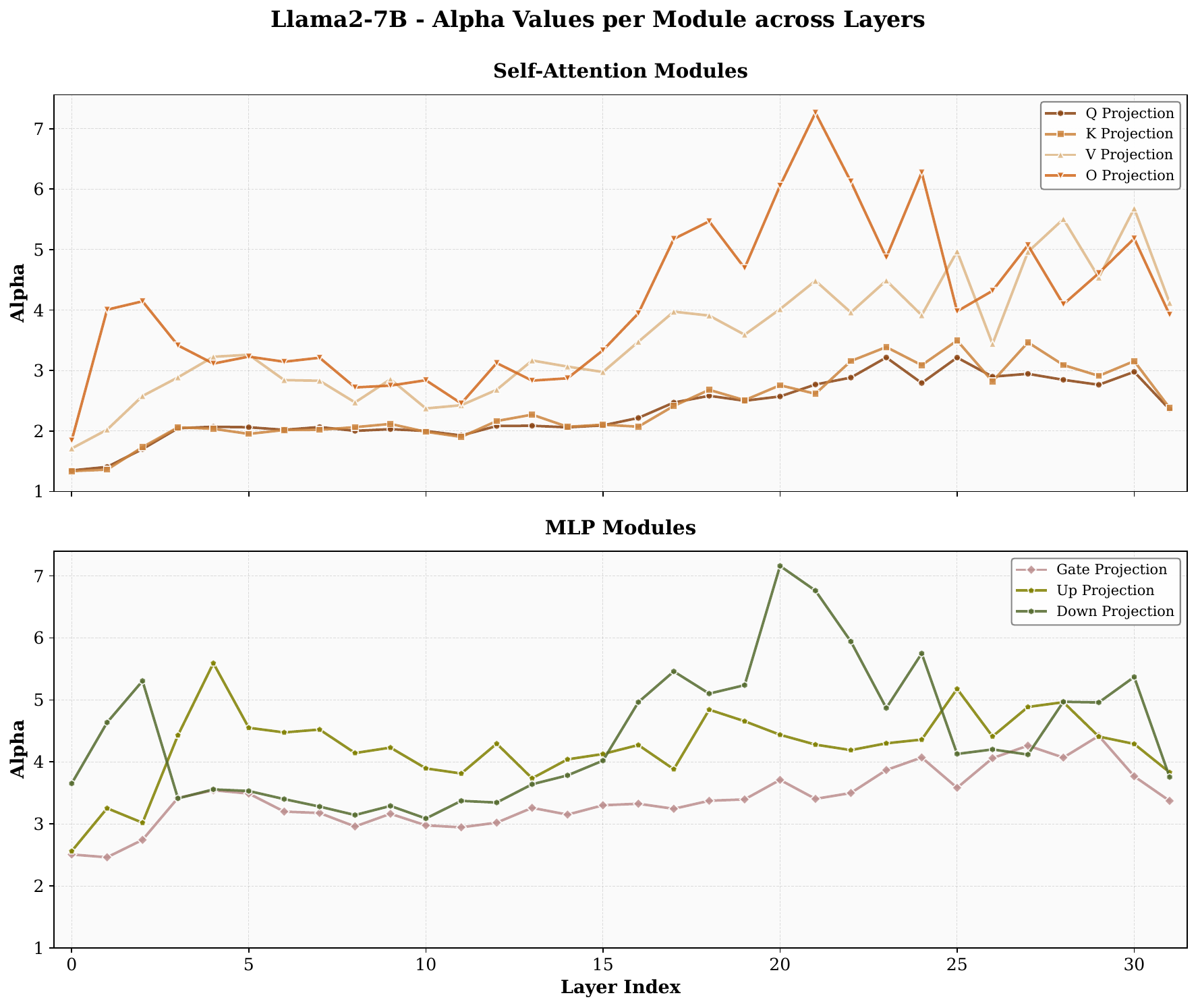}
    \caption{Power-law exponents $\alpha$ across layers in LLaMA2-7B. Smaller values indicate heavier tails.}
    \label{fig:alpha_across_layers}
\end{figure}

\paragraph{Heavy-tailed Metric.}
To quantify the heavy-tailedness of ESDs, HT-SR adopts a \emph{shape metric} based on a power-law fit. Given the ESD of the correlation matrix $\mathbf X^{(\ell,p)}$, we fit a power-law (PL) on the interval $(\lambda_{\min}^{(\ell,p)},\lambda_{\max}^{(\ell,p)})$, taking the form
\begin{equation}
p(\lambda)\propto \lambda^{-\alpha^{(\ell,p)}}, 
\qquad 
\lambda_{\min}^{(\ell,p)}<\lambda<\lambda_{\max}^{(\ell,p)}.
\end{equation}

Here, $p(\lambda)$ denotes the density of eigenvalues $\lambda$ within the specified interval. The exponent $\alpha^{(\ell,p)}$ provides a quantitative measure of the ESD’s tail heaviness. In practice, we compute $\alpha^{(\ell,p)}$ for each layer $\ell$ and module $p$.

For estimating $\alpha$, we apply the Hill estimator \citep{8b55684a-14a0-31e4-b039-d044b4625cb8,xiao2023heavytailedregularizationweightmatrices,zhou2023temperature}, denoted as $\mathrm{PL\_Alpha\_Hill}$,

\begin{equation}
\mathrm{PL\_Alpha\_Hill}
=1+\frac{k}{\sum_{i=1}^{k}\ln \left(\frac{\lambda_{n-i+1}}{\lambda_{n-k}}\right)},
\end{equation}

As shown in Figure~\ref{fig:alpha_across_layers}, the estimated $\alpha$ values vary across both depth and projection types, indicating heterogeneous heavy-tailed behavior.

\subsection{ActTail: Global Activation Sparsity Allocation}
Motivated by the observed spectral diversity across projections, ActTail therefore uses each projection’s heavy-tailed exponent $\alpha^{(\ell,p)}$ to set its TopK input sparsity ratio $s^{(\ell,p)}$. ActTail computes $\alpha^{(\ell,p)}$ from the ESD of the correlation matrix $\mathbf X^{(\ell,p)}=\mathbf W^{(\ell,p)\top}\mathbf W^{(\ell,p)}$ for each projection weight matrix $\mathbf W^{(\ell,p)}$. Then, we map the heavy-tailed exponents $\alpha^{(\ell,p)}$ of all projections to
projection-level sparsity ratios through a monotone affine function. Figure~\ref{fig:acttail_workflow} provides an overview of the ActTail workflow, from estimating heavy-tailed exponents to allocating projection-level sparsity and applying TopK gating during inference.

\begin{equation}
s^{(\ell,p)}
=
\eta\left[
\frac{\alpha^{(\ell,p)}-\alpha_{\min}}
{\alpha_{\max}-\alpha_{\min}}(s_2-s_1)+s_1
\right],
\label{eq:acttail_map}
\end{equation}
where $\alpha_{\min}$ and $\alpha_{\max}$ are the minimum and maximum values of
$\alpha^{(\ell,p)}$ across all projections, and $(s_1,s_2)$ controls the degree
of non-uniformity. The normalization factor $\eta$ rescales all projection sparsities to satisfy a target global sparsity level $S$, i.e.,
$\sum_{(\ell,p)} s^{(\ell,p)} d^{(\ell,p)} = S \sum_{(\ell,p)} d^{(\ell,p)}$,
where $d^{(\ell,p)}$ denotes the parameter count of $\mathbf W^{(\ell,p)}$.

\section{Theoretical Analysis}
\begin{figure*}[t]
  \centering
  \includegraphics[width=0.95\textwidth]{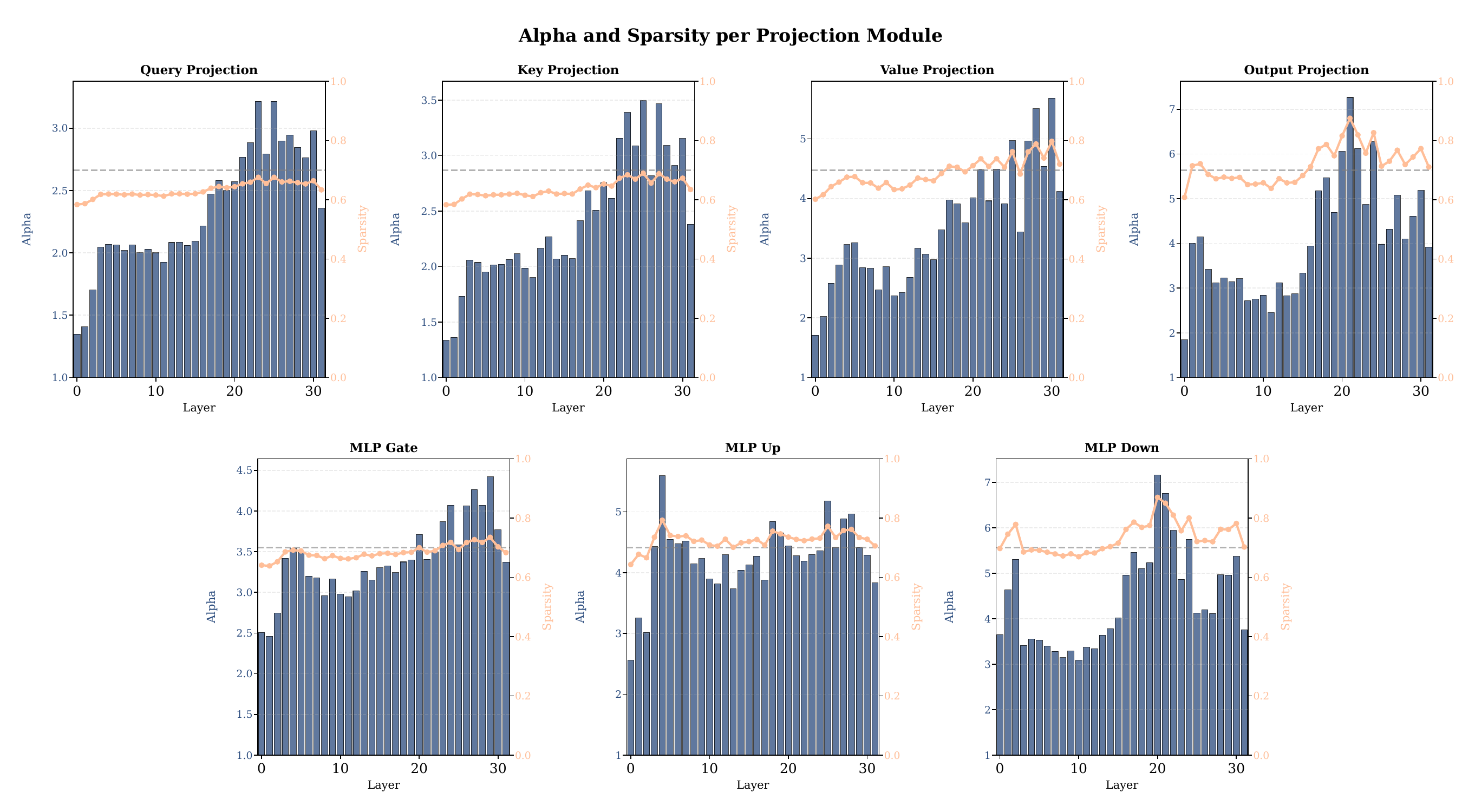}
  \caption{Alpha and sparsity per projection module across layers.}
  \label{fig:appendix-alpha-sparsity}
\end{figure*}

Section~3 motivates ActTail from the perspective of Heavy-Tailed Self-Regularization (HT-SR) theory and introduces a linear mapping from projection-level heavy-tail exponents to sparsity ratios. However, it lacks a theoretical explanation why projections with more heavy-tailed spectra should admit more aggressive TopK activation sparsification. In this section, we provide a rigorous theoretical analysis establishing a formal link between heavy-tailed spectral behavior and activation sparsity under TopK sparsity. 

To prove that heavy-tailed spectra lead to activation sparsity, we first establish a generic TopK approximation guarantee.
Lemma~\ref{lem:stechkin_topk} provides an $\ell_2$ best $K$-term (TopK) approximation bound in terms of the weak-$\ell_p$
(Lorentz) quasi-norm, and will serve as the first step of our proof chain.

\begin{lemma}[$\ell_2$ Best $K$-term approximation]
\label{lem:stechkin_topk}
Let $d\ge 1$ and $p\in(0,2)$. 
For $y\in\mathbb{R}^d$, let $|y|_{(1)}\ge\cdots\ge |y|_{(d)}$ denote the decreasing rearrangement of
$\{|y_1|,\ldots,|y_d|\}$.
Let $T_K(y)$ be the TopK truncation that keeps the $K$ largest-magnitude coordinates of $y$ and sets the rest to zero. Then for any $K\in\{1,\dots,d\}$,
\begin{equation}
\label{eq:stechkin}
\|y-T_K(y)\|_2  \le  C_p   \|y\|_{p,\infty}  K^{\frac12-\frac1p},
\end{equation}
where $C_p>0$. $\|\cdot\|_{p,\infty}$ is the weak-$\ell_p$ (Lorentz) quasi-norm $\|y\|_{p,\infty}:=\sup_{1\le j\le d}   j^{1/p} |y|_{(j)}$.
\end{lemma}

The proof of Lemma~\ref{lem:stechkin_topk} is provided in Appendix~\ref{app:proof_stechkin}. Lemma~\ref{lem:stechkin_topk} quantifies how well a vector $y$ can be approximated in $\ell_2$ by its TopK truncation $T_K(y)$, with the approximation quality controlled by the weak-$\ell_p$ quasi-norm $\|y\|_{p,\infty}$. The theorem upper-bounds the $\ell_2$ error of TopK truncation by
$\|y-T_K(y)\|_2 \le C_p \|y\|_{p,\infty} K^{\frac12-\frac1p}$.
In particular, if $y\in\ell_{p,\infty}$ (i.e., $\|y\|_{p,\infty}<\infty$), then the TopK error in $\ell_2$
decays as $K^{1/2-1/p}$.

To apply Lemma~\ref{lem:stechkin_topk}, we next relate the weak-$\ell_p$ behavior of $y=Wx$ to the heavy-tail ESD exponent $\alpha$ of $X=W^\top W$ via spectral energy concentration. Lemma~\ref{lem:karamata_energy_capture} connects the fitted heavy-tail exponent $\alpha$ of the ESD to a concrete
energy-capture law for the spectrum.
It shows how the fraction of spectral energy contained in the top-$qn$ eigenvalues scales with $q$.

\begin{lemma}[power-law tail integral]
\label{lem:karamata_energy_capture}
Let $X\in\mathbb{R}^{n\times n}$ be positive semi-definite matrix (PSD) with eigenvalues $\lambda_1\ge\cdots\ge\lambda_n\ge 0$.
Assume the ESD of $X$ follows a power-law tail with exponent $\alpha>2$, there exist constants $c_0,c_1>0$ and a smooth function $L(\cdot)$ such that for large $t$,
 \begin{equation}
\mathbb{P}(\Lambda\ge t) =  c_0  t^{1-\alpha} L(t)
\qquad (\Lambda \sim \text{ESD of }X),
 \end{equation}
Moreover, Define the \emph{top-$q$ spectral energy capture ratio} for $q\in(0,1)$:
 \begin{equation}
R(q) := \frac{\sum_{j\le qn}\lambda_j}{\sum_{j=1}^n \lambda_j}.
 \end{equation}
Then, for $q$ in the regime where the power-law fit holds, one has the scaling law
\begin{equation}
\label{eq:Rq_scaling}
R(q) \asymp  q^{\frac{\alpha-2}{\alpha-1}},
\end{equation}
\end{lemma}

Lemma~\ref{lem:karamata_energy_capture} states that under a power-law ESD tail with exponent $\alpha>2$, the top-$q$ spectral energy capture ratio obeys $R(q)\asymp q^{\frac{\alpha-2}{\alpha-1}}$ in the regime where the power-law fit holds. Since $\frac{\alpha-2}{\alpha-1}$ is increasing in $\alpha$, for any fixed $q\in(0,1)$, a smaller $\alpha$ corresponds to a larger $R(q)$ and hence stronger concentration of spectral energy in the leading eigenvalues. The proof of Lemma~\ref{lem:karamata_energy_capture} is provided in Appendix~\ref{app:proof_karamata}.

Building on this $\alpha$-dependent spectral concentration, we use the heavy-tail exponent as a projection-level measure of structural sensitivity. A smaller $\alpha$ indicates that a larger fraction of the projection's spectral energy is concentrated in its leading spectral directions, whereas a larger $\alpha$ corresponds to a flatter spectral-energy profile. This motivates a conservative sparsification principle: projections with smaller $\alpha$ should retain more activation coordinates, while projections with larger $\alpha$ can be assigned larger sparsity budgets. Accordingly, we use $\alpha$ to guide projection-specific sparsity allocation, with sparsity increasing monotonically with $\alpha$.

\begin{table*}[t]
\centering

\small
\setlength{\tabcolsep}{1.5pt}
\begin{tabular}{l|c|l|c c c c c c c}
\toprule
\hline
\textbf{Model} & \textbf{Sparsity} & \textbf{Method} &
\textbf{Avg.} & \textbf{MMLU} & \textbf{ARC-c} & \textbf{HellaSwag} & \textbf{BOOLQ} & \textbf{PIQA} & \textbf{WinoGrande} \\
\hline
\multirow{5}{*}{LLaMA2-7B}
& 0\% & Dense
& 62.72 & 45.84 & 43.43 & 57.13 & 77.74 & 78.07 & 74.11 \\
& 70\% & \textsc{TEAL}$_{\text{Top-}K}$
& 56.02 & 26.53 & 28.50 & 41.84 & 65.66 & 70.67 & 50.36 \\
& 70\% & ActTail
& \textbf{58.06} & \textbf{27.35} & \textbf{30.97} & \textbf{44.07} & \textbf{66.97} & \textbf{72.25} & \textbf{56.12} \\
\cline{2-10}
& 80\% & \textsc{TEAL}$_{\text{Top-}K}$
& 47.87 & 23.58 & 19.62 & 28.97 & 44.62 & 60.50 & 48.78 \\
& 80\% & ActTail
& \textbf{49.27} & \textbf{25.32} & \textbf{21.08} & \textbf{30.09} & \textbf{56.94} & \textbf{61.92} & \textbf{51.22} \\
\hline
\hline
\multirow{5}{*}{LLaMA2-13B}
& 0\% & Dense
& 65.93 & 55.13 & 48.46 & 60.06 & 80.61 & 79.05 & 72.22 \\
& 70\% & \textsc{TEAL}$_{\text{Top-}K}$
& 58.06 & 30.63 & 33.96 & 46.51 & 68.23 & 72.80 & 61.40 \\
& 70\% & ActTail
& \textbf{62.18} & \textbf{36.88} & \textbf{40.27} & \textbf{49.55} & \textbf{72.51} & \textbf{74.10} & \textbf{63.77} \\
\cline{2-10}
& 80\% & \textsc{TEAL}$_{\text{Top-}K}$
& 48.73 & 24.95 & 20.22 & 30.62 & 51.47 & \textbf{63.38} & 51.78 \\
& 80\% & ActTail
& \textbf{52.61} & \textbf{26.61} & \textbf{25.51} & \textbf{36.23} & \textbf{64.16} & 63.22 & \textbf{52.93} \\
\hline
\hline
\multirow{5}{*}{Mistral-7B}
& 0\% & Dense
& 69.72 & 62.65 & 50.34 & 61.24 & 83.64 & 80.58 & 73.80 \\
& 70\% & \textsc{TEAL}$_{\text{Top-}K}$
& 59.17 & 30.37 & 34.30 & \textbf{47.17} & \textbf{67.80} & \textbf{73.23} & 62.67 \\
& 70\% & ActTail
& \textbf{59.42} & \textbf{33.20} & \textbf{34.39} & 46.78 & 67.49 & \textbf{73.23} & \textbf{64.25} \\
\cline{2-10}
& 80\% & \textsc{TEAL}$_{\text{Top-}K}$
& 48.12 & 24.37 & 18.77 & 31.64 & 44.46 & 63.60 & \textbf{51.46} \\
& 80\% & ActTail
& \textbf{49.15} & \textbf{24.70} & \textbf{21.16} & \textbf{32.10} & \textbf{51.72} & \textbf{63.87} & 50.91 \\
\hline
\bottomrule
\end{tabular}%
\caption{Downstream performance under different sparsity levels.
Best results at each sparsity level are in \textbf{bold}.}
\label{tab:downstream_all}
\end{table*}

\begin{assumption}[Spectral-to-coordinate regularity]
\label{assump:hw_bridge}
Fix $(\ell,p)$ and let $W:=W^{(\ell,p)}\in\mathbb{R}^{d_{\text{out}}\times d_{\text{in}}}$ and $X:=W^\top W$. Let $x\in\mathbb{R}^{d_{\text{in}}}$ be mean-zero and coordinate-wise sub-Gaussian with parameter $\kappa$, and set $y:=Wx$. Assume the ESD of $X$ satisfies the condition in Lemma~\ref{lem:karamata_energy_capture} with exponent $\alpha^{(\ell,p)}=\alpha>2$. Let $v_i:=\mathrm{Var}(x_i)$ and assume $\mathbb{P}(|x_i|\ge t)\le 2\exp(-ct^2/(\kappa^2v_i))$ for some universal constant $c>0$. We assume that the decreasing rearrangement of the coordinate variances satisfies
\begin{equation}
\label{eq:diag_spectrum_comparable}
v_{(j)} \le \gamma j^{-(\alpha-1)}, \qquad \forall j\in\{1,\dots,d_{\text{in}}\},
\end{equation}
for some constant $\gamma>0$.
\end{assumption}

\begin{theorem}[Spectral energy concentration]
\label{thm:hw_bridge}
Under Assumption~\ref{assump:hw_bridge}, for any $\delta\in(0,1)$, with probability at least $1-\delta$, the input activation vector $x$ belongs to a weak-$\ell_{p(\alpha)}$ ball with $p(\alpha)=\frac{2}{\alpha-1}\in(0,2)$, in the sense that
\begin{equation}
\label{eq:weaklp_y}
|x|_{(j)} \le C(\kappa,\gamma)\sqrt{\log\Big(\frac{2d_{\text{in}}}{\delta}\Big)}\,j^{-1/p(\alpha)}, \qquad \forall j\in\{1,\dots,d_{\text{in}}\}.
\end{equation}
Consequently, by Lemma~\ref{lem:stechkin_topk} and $\|Wx-WT_K(x)\|_2\le\|W\|_2\|x-T_K(x)\|_2$, for any target tolerance $\varepsilon>0$, choosing
\begin{equation}
\label{eq:K_alpha_eps}
K(\alpha,\varepsilon)\geq \left(\frac{\|W\|_2 C(\kappa,\gamma)\sqrt{\log(2d_{\text{in}}/\delta)}}{\varepsilon\sqrt{\alpha-2}}\right)^{\frac{2}{\alpha-2}}
\end{equation}
ensures $\|Wx-WT_{K(\alpha,\varepsilon)}(x)\|_2\le\varepsilon$, provided $K(\alpha,\varepsilon)\le d_{\text{in}}$, and the corresponding activation sparsity ratio is $s(\alpha,\varepsilon)=1-\frac{K(\alpha,\varepsilon)}{d_{\text{in}}}$. If the required budget exceeds $d_{\text{in}}$, retaining all input coordinates yields zero approximation error.
\end{theorem}

\begin{figure*}[!ht]
  \centering
  \includegraphics[width=0.95 \textwidth]{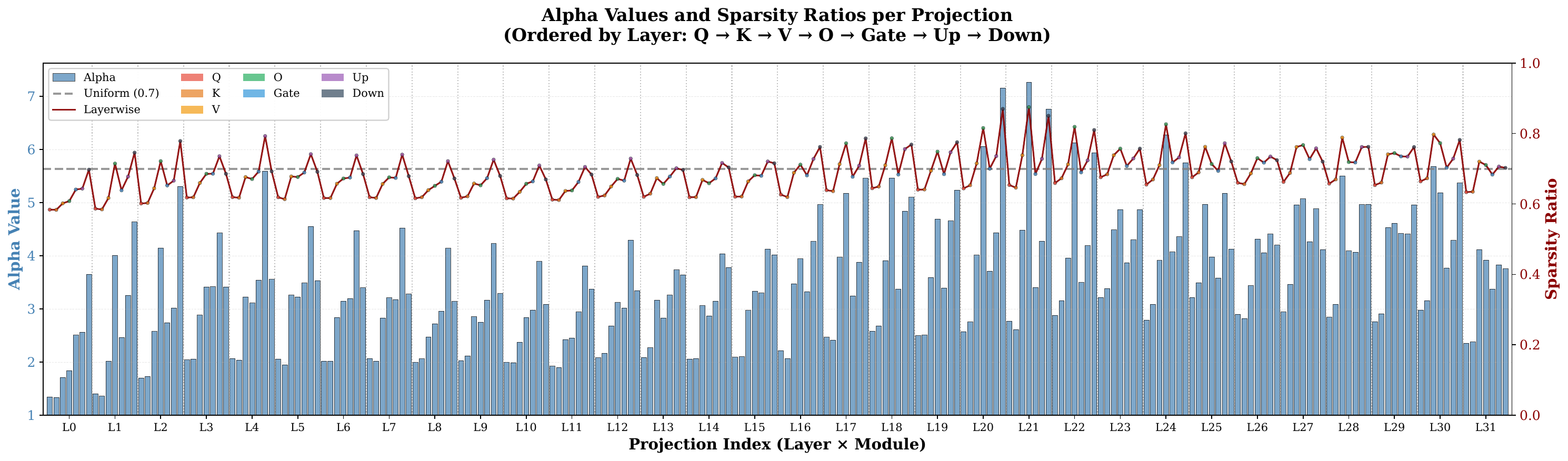}
    
  \caption{Alpha values and sparsity ratios per projection, ordered by layer and module type (Q, K, V, O, Gate, Up, Down).}
  \label{fig:alpha-sparsity}
  
\end{figure*}

Theorem~\ref{thm:hw_bridge} establishes an explicit connection between the heavy-tail exponent $\alpha$ and input TopK activation sparsification. Importantly, the bound directly characterizes the output perturbation induced by input activation truncation, $\|Wx-WT_K(x)\|_2$, rather than output activation truncation, $\|Wx-T_K(Wx)\|_2$. From Eq.~\eqref{eq:K_alpha_eps}, a smaller $\alpha$ yields a slower TopK error decay and requires retaining more input activation coordinates under a fixed error tolerance when the remaining constants are controlled, corresponding to a smaller sparsity ratio. Together with Lemma~\ref{lem:karamata_energy_capture}, this motivates our $\alpha$-guided allocation: projections with smaller $\alpha$ are sparsified more conservatively, whereas those with larger $\alpha$ receive larger sparsity budgets. The proof is provided in Appendix~\ref{app:proof_hw_bridge}.

\section{Results}

\subsection{Experimental Settings}
\paragraph{Models and Evaluation.} 
We evaluate our approach on several decoder-based LLMs, including LLaMA2 7B and 13B \citep{touvron2023llama}, Mistral 7B v0.1 \citep{jiang2023mistral}, and Qwen1.5-7B\citep{qwen1.5}. Language modeling performance is measured by perplexity on WikiText2 \citep{merity2016pointer}. Downstream performance is evaluated using the EleutherAI LM Evaluation Harness \citep{eval-harness} on an aggregate of six tasks: 5-shot MMLU, 0-shot ARC Challenge, 0-shot HellaSwag, 0-shot PIQA, 0-shot BOOLQ and 10-shot WinoGrande \citep{hendrycks2021measuringmassivemultitasklanguage,clark2018thinksolvedquestionanswering, zellers2019hellaswagmachinereallyfinish, bisk2019piqareasoningphysicalcommonsense, sakaguchi2019winograndeadversarialwinogradschema,clark2019boolq}\footnote{Our code is available at https://anonymous.4open.science/r/ActTail-76D0/.}.

\paragraph{Baselines.}
We primarily compare ActTail against TEAL(ICLR'2025)\citep{liu2025trainingfree}, a recent state-of-the-art activation sparsity method,which is traning free without any finetuning. To ensure a fair and controlled comparison, we use its TopK formulation with uniform sparsity allocation across projections. We report model-level sparsity for all methods. We do not include TEAL’s threshold-based version, as its input-dependent sparsity leads to varying effective sparsity levels across inputs, which prevents a controlled comparison at matched sparsity budgets.

\subsection{Main Results}
\paragraph{Language Modeling.}
Table~\ref{tab:ppl_main} reports perplexity on WikiText2 for LLaMA2, Mistral, and Qwen1.5 under different activation sparsity levels.
ActTail consistently outperforms uniform Top-$K$ allocation across all models, with the advantage becoming increasingly pronounced as sparsity grows.
At 70\% sparsity, ActTail reduces perplexity by 13.5\% on LLaMA2-7B, 9.2\% on LLaMA2-13B, 1.8\% on Mistral-7B, and 7.2\% on Qwen1.5-7B relative to uniform allocation.
The gains further amplify at 80\% sparsity, where ActTail achieves perplexity reductions of 21.8\% on LLaMA2-7B, 40.1\% on LLaMA2-13B, and 9.4\% on Mistral-7B, while still providing a consistent 1.8\% improvement on Qwen1.5-7B.
These results demonstrate that heavy-tail-guided projection-level sparsity allocation is particularly effective in the high-sparsity regime, where uniform heuristics incur substantial degradation.

\paragraph{Downstream Tasks.}
ActTail consistently improves downstream accuracy over uniform activation sparsity on LLaMA2-7B, LLaMA2-13B and Mistral-7B across a diverse set of reasoning and knowledge benchmarks, including MMLU, ARC-c, HellaSwag, BOOLQ, PIQA and WinoGrande(Table~\ref{tab:downstream_all}).  At 70\% sparsity, ActTail improves the averaged score by 3.6\% on LLaMA2-7B and by 7.1\% on LLaMA2-13B relative to uniform Top-K sparsification. Similar performance gains are also observed on Mistral-7B. The performance advantage of ActTail remains evident at 80\% sparsity. At 80\% sparsity, ActTail improves the averaged performance by 2.9\% on LLaMA2-7B, 8.0\% on LLaMA2-13B, and 2.1\% on Mistral-7B relative to uniform ones. Among all evaluated benchmarks, ActTail exhibits the most pronounced improvements on BOOLQ. 
Specifically, ActTail improves BOOLQ accuracy by 27.6\% on LLaMA2-7B, 24.6\% on LLaMA2-13B, and 16.3\% on Mistral-7B at 80\% sparsity. Overall, these results confirm that leveraging heavy-tailed spectral heterogeneity enables more robust activation sparsity, yielding superior accuracy and efficiency trade-offs across model sizes and task types.

\subsection{Empirical Analysis of Heavy-Tailed Metrics and Activation Sparsity}
To better understand how the heavy-tailed self-regularization (HT-SR) metric influences activation sparsity decisions in practice, we analyze the relationship between the estimated heavy-tail exponent $\alpha$ and the resulting sparsity ratios across projections. Figure~\ref{fig:alpha-sparsity} provides a global view of projection-level $\alpha$ and sparsity ratios, while Figure~\ref{fig:appendix-alpha-sparsity} groups them by module type(Q, K, V, O, Gate, Up, Down) for module-wise comparison. Several clear patterns emerge from this analysis.

First, projections with smaller $\alpha$ values are assigned lower activation sparsity, as heavier-tailed projections encode more concentrated representations that are less tolerant to aggressive sparsification. Second, sparsity allocation varies systematically across module types. Attention projections, especially Q and O, tend to have smaller $\alpha$ values and receive lower sparsity, whereas MLP projections (Gate, Up, Down) show larger $\alpha$ values and tolerate higher sparsity. This explains the suboptimality of uniform or layerwise sparsity schemes. Finally, sparsity closely follows the variation of $\alpha$ across depth rather than remaining constant, indicating that the proposed allocation adapts across both layers and projection types.

\section{Conclusion}
ActTail uses Heavy-Tailed Self-Regularization theory to show that different projections have different spectral behavior, and uses this to set TopK input activation sparsity per projection. We give theory linking heavy-tailed spectral statistics to projection sparsifiability, explaining why heavier-tailed projections can be sparsified more aggressively with limited loss. Experiments on multiple large language model families show consistent gains in perplexity and downstream performance at high sparsity, supporting HT-SR as a practical basis for theory-guided activation sparsity.

\section{Limitations}
ActTail depends on estimating the heavy-tail exponent \(\alpha\), which can vary with the fitting range and spectral precision. Its affine sparsity mapping may not be optimal for all models or tasks. We focus on training-free TopK sparsity during decoding, while calibration methods, prefill acceleration, sparse kernels, and deviations from idealized theory remain future work.

\bibliography{example_paper}

\appendix

\section{Mathematical Background and Proofs}
\subsection{Proof of Lemma~\ref{lem:stechkin_topk}}
\label{app:proof_stechkin}

In this section, we provide the proof of Lemma~\ref{lem:stechkin_topk}, which gives the $\ell_2$ approximation error of TopK truncation for vectors in weak-$\ell_p$ spaces.

\begin{proof}
By the definition of the weak-$\ell_p$ quasi-norm,
\begin{equation}
\label{eq:stechkin_order_bound}
|y|_{(j)}\le \|y\|_{p,\infty}j^{-1/p}.
\end{equation}
Using Eq.~\eqref{eq:stechkin_order_bound} and $p\in(0,2)$, we obtain
\begin{equation}
\label{eq:stechkin_tail_derivation}
\begin{aligned}
\|y-T_K(y)\|_2^2
&=\sum_{j>K}|y|_{(j)}^2\\
&\le \|y\|_{p,\infty}^2\sum_{j>K}j^{-2/p}\\
&\le \|y\|_{p,\infty}^2\int_K^\infty t^{-2/p}\,dt\\
&=\frac{p}{2-p}\|y\|_{p,\infty}^2K^{1-2/p}.
\end{aligned}
\end{equation}
Taking square roots on both sides of Eq.~\eqref{eq:stechkin_tail_derivation} gives
\begin{equation}
\label{eq:stechkin_final_derivation}
\begin{aligned}
\|y-T_K(y)\|_2
&\le \sqrt{\frac{p}{2-p}}\|y\|_{p,\infty}K^{\frac12-\frac1p}\\
&=C_p\|y\|_{p,\infty}K^{\frac12-\frac1p},
\end{aligned}
\end{equation}
where $C_p=\sqrt{\frac{p}{2-p}}$. Equation~\eqref{eq:stechkin_final_derivation} is exactly Eq.~\eqref{eq:stechkin}, which completes the proof.
\end{proof}

\subsection{Proof of Lemma~\ref{lem:karamata_energy_capture}}
\label{app:proof_karamata}
In this section, we provide the proof of Lemma~\ref{lem:karamata_energy_capture}. The result characterizes how spectral energy concentrates under a power-law tail of the empirical spectral density (ESD), and follows from standard arguments in heavy-tailed analysis and integral approximations.

\begin{proof}
Let $\bar F(t):=\mathbb{P}(\Lambda\ge t)$. By assumption,
\begin{equation}
\label{eq:karamata_tail}
\bar F(t)=c_0t^{1-\alpha}L(t),
\end{equation}
where $L$ is slowly varying. Thus, $\bar F$ is regularly varying with index $1-\alpha$. Let $\tau_q$ be the upper-tail quantile defined by
\begin{equation}
\label{eq:karamata_quantile_def}
\bar F(\tau_q)=q.
\end{equation}
Combining Eqs.~\eqref{eq:karamata_tail} and~\eqref{eq:karamata_quantile_def}, and applying standard inversion results for regularly varying functions~\cite{bingham1987regular}, we obtain
\begin{equation}
\label{eq:karamata_quantile_scaling}
\begin{aligned}
q
&=\bar F(\tau_q)\\
&=c_0\tau_q^{1-\alpha}L(\tau_q),\\
\tau_q
&=q^{-\frac{1}{\alpha-1}}\widetilde L\Big(\frac{1}{q}\Big)\\
&\asymp q^{-\frac{1}{\alpha-1}},
\end{aligned}
\end{equation}
up to slowly varying factors, where $\widetilde L$ is slowly varying.

Now consider the upper-tail first moment
\begin{equation}
\label{eq:karamata_first_moment}
M_1(t):=\mathbb{E}\big[\Lambda\mathbf{1}\{\Lambda\ge t\}\big].
\end{equation}
By Karamata's theorem~\cite{resnick1987extreme}, Eq.~\eqref{eq:karamata_tail} and $\alpha>2$ imply
\begin{equation}
\label{eq:karamata_first_moment_scaling}
M_1(t)\asymp t^{2-\alpha}L(t).
\end{equation}
Combining Eqs.~\eqref{eq:karamata_quantile_scaling}, \eqref{eq:karamata_first_moment}, and~\eqref{eq:karamata_first_moment_scaling}, we obtain, up to slowly varying factors,
\begin{equation}
\label{eq:karamata_energy_scaling}
\begin{aligned}
\frac{\mathbb{E}\big[\Lambda\mathbf{1}\{\Lambda\ge\tau_q\}\big]}{\mathbb{E}[\Lambda]}
&=\frac{M_1(\tau_q)}{\mathbb{E}[\Lambda]}\\
&\asymp \tau_q^{2-\alpha}L(\tau_q)\\
&\asymp \left(q^{-\frac{1}{\alpha-1}}\right)^{2-\alpha}\\
&=q^{\frac{\alpha-2}{\alpha-1}}.
\end{aligned}
\end{equation}
Finally, Eq.~\eqref{eq:karamata_energy_scaling} transfers, under the empirical quantile approximation, to the discrete top-$q$ spectral energy ratio $R(q)$, yielding Eq.~\eqref{eq:Rq_scaling} and completing the proof.
\end{proof}

\subsection{Proof of Theorem~\ref{thm:hw_bridge}}
\label{app:proof_hw_bridge}
In this section, we prove Theorem~\ref{thm:hw_bridge}, which links heavy-tailed spectral behavior to the compressibility of input activations under the spectral-to-coordinate regularity condition in Assumption~\ref{assump:hw_bridge}.

\begin{proof}
Since $x$ is mean-zero and coordinate-wise sub-Gaussian with parameter $\kappa$, each input coordinate $x_i$ satisfies the tail bound in Assumption~\ref{assump:hw_bridge}, with variance $\mathrm{Var}(x_i)=v_i$. Hence, for some universal constant $c>0$,
\begin{equation}
\label{eq:subgaussian_coordinate_tail}
\mathbb{P}\big(|x_i|\ge t\big)\le 2\exp\left(-c\frac{t^2}{\kappa^2v_i}\right).
\end{equation}

Setting $t_i=\kappa\sqrt{\frac{v_i}{c}\log\left(\frac{2d_{\text{in}}}{\delta}\right)}$ and substituting it into Eq.~\eqref{eq:subgaussian_coordinate_tail} gives
\begin{equation}
\label{eq:coordinate_failure_prob}
\begin{aligned}
\mathbb{P}\big(|x_i|\ge t_i\big)
&\le 2\exp\left(-c\frac{t_i^2}{\kappa^2v_i}\right)\\
&=2\exp\left(-\log\left(\frac{2d_{\text{in}}}{\delta}\right)\right)\\
&=\frac{\delta}{d_{\text{in}}}.
\end{aligned}
\end{equation}

Applying a union bound to Eq.~\eqref{eq:coordinate_failure_prob} over all input coordinates yields
\begin{equation}
\label{eq:coordinate_union_bound}
\begin{aligned}
&\mathbb{P}\left(\exists i\in[d_{\text{in}}]:|x_i|\ge t_i\right)\\
&\le \sum_{i=1}^{d_{\text{in}}}\mathbb{P}\big(|x_i|\ge t_i\big)\\
&\le \sum_{i=1}^{d_{\text{in}}}\frac{\delta}{d_{\text{in}}}=\delta.
\end{aligned}
\end{equation}

Therefore, by Eq.~\eqref{eq:coordinate_union_bound}, with probability at least $1-\delta$, simultaneously for all $i$,
\begin{equation}
\label{eq:coordinate_high_prob_bound}
|x_i|\le C_0\kappa\sqrt{v_i\log\left(\frac{2d_{\text{in}}}{\delta}\right)},
\end{equation}
where $C_0=c^{-1/2}>0$ is a universal constant. By Assumption~\ref{assump:hw_bridge}, Eq.~\eqref{eq:diag_spectrum_comparable} gives $v_{(j)}\le\gamma j^{-(\alpha-1)}$. Combining this with Eq.~\eqref{eq:coordinate_high_prob_bound} and sorting the coordinates by magnitude gives
\begin{equation}
\label{eq:activation_order_decay}
\begin{aligned}
|x|_{(j)}
&\le C_0\kappa\sqrt{v_{(j)}\log\left(\frac{2d_{\text{in}}}{\delta}\right)}\\
&\le C_0\kappa\sqrt{\gamma\log\left(\frac{2d_{\text{in}}}{\delta}\right)}\,j^{-\frac{\alpha-1}{2}}\\
&=C(\kappa,\gamma)\sqrt{\log\left(\frac{2d_{\text{in}}}{\delta}\right)}\,j^{-1/p(\alpha)},
\end{aligned}
\end{equation}
where
\begin{equation}
\label{eq:p_alpha_definition}
p(\alpha)=\frac{2}{\alpha-1}\in(0,2),
\end{equation}
and $C(\kappa,\gamma)=C_0\kappa\sqrt{\gamma}$.

It follows from Eq.~\eqref{eq:activation_order_decay} that
\begin{equation}
\label{eq:weak_lp_bound}
\|x\|_{p(\alpha),\infty}\le C(\kappa,\gamma)\sqrt{\log\left(\frac{2d_{\text{in}}}{\delta}\right)}.
\end{equation}

Applying Lemma~\ref{lem:stechkin_topk} with $p=p(\alpha)$ and using Eqs.~\eqref{eq:p_alpha_definition} and~\eqref{eq:weak_lp_bound}, we obtain
\begin{equation}
\label{eq:topk_alpha_error}
\begin{aligned}
&\|Wx-WT_K(x)\|_2\\
&=\|W(x-T_K(x))\|_2\\
&\le \|W\|_2\|x-T_K(x)\|_2\\
&\le \|W\|_2 C_{p(\alpha)}\|x\|_{p(\alpha),\infty}K^{\frac12-\frac{1}{p(\alpha)}}\\
&\le \|W\|_2 C(\kappa,\gamma)C_{p(\alpha)}\sqrt{\log\left(\frac{2d_{\text{in}}}{\delta}\right)}\,K^{-\frac{\alpha-2}{2}}.
\end{aligned}
\end{equation}

For the weak-$\ell_p$ Stechkin inequality, the constant can be chosen as
\begin{equation}
\label{eq:stechkin_constant_explicit}
C_p=\sqrt{\frac{p}{2-p}}.
\end{equation}
Substituting $p(\alpha)=2/(\alpha-1)$ yields
\begin{equation}
\label{eq:stechkin_constant_alpha}
C_{p(\alpha)}=\sqrt{\frac{2/(\alpha-1)}{2-2/(\alpha-1)}}=\frac{1}{\sqrt{\alpha-2}}.
\end{equation}

Consequently, Eq.~\eqref{eq:topk_alpha_error} becomes
\begin{equation}
\label{eq:topk_alpha_error_explicit}
\begin{aligned}
&\|Wx-WT_K(x)\|_2\\
&\le \frac{\|W\|_2 C(\kappa,\gamma)}{\sqrt{\alpha-2}}\sqrt{\log\left(\frac{2d_{\text{in}}}{\delta}\right)}\,K^{-\frac{\alpha-2}{2}}.
\end{aligned}
\end{equation}

Finally, requiring the right-hand side of Eq.~\eqref{eq:topk_alpha_error_explicit} to be at most $\varepsilon$ gives
\begin{equation}
\label{eq:K_alpha_eps_derivation}
\begin{aligned}
&\frac{\|W\|_2 C(\kappa,\gamma)}{\sqrt{\alpha-2}}\sqrt{\log\left(\frac{2d_{\text{in}}}{\delta}\right)}\,K^{-\frac{\alpha-2}{2}}\le\varepsilon,\\
&K^{\frac{\alpha-2}{2}}\ge\frac{\|W\|_2 C(\kappa,\gamma)\sqrt{\log(2d_{\text{in}}/\delta)}}{\varepsilon\sqrt{\alpha-2}},\\
&K\ge\left(\frac{\|W\|_2 C(\kappa,\gamma)\sqrt{\log(2d_{\text{in}}/\delta)}}{\varepsilon\sqrt{\alpha-2}}\right)^{\frac{2}{\alpha-2}}.
\end{aligned}
\end{equation}

This is exactly Eq.~\eqref{eq:K_alpha_eps}. If the required budget exceeds $d_{\text{in}}$, retaining all input coordinates gives $T_{d_{\text{in}}}(x)=x$ and hence zero output approximation error. The corresponding activation sparsity ratio is $s(\alpha,\varepsilon)=1-K(\alpha,\varepsilon)/d_{\text{in}}$, completing the proof.
\end{proof}

\end{document}